\documentclass[final]{cvpr}

\usepackage{times}
\usepackage{epsfig}
\usepackage{graphicx}
\usepackage{amsmath}
\usepackage{amssymb}

\usepackage{amsthm}

\usepackage{float}
\usepackage{multirow}

\def\R{{\mathbb R}}
\def\Diff{{\operatorname{Diff}}}

\usepackage{url}

\newtheorem{theorem}{Theorem}

\usepackage[normalem]{ulem}

\usepackage[pagebackref=true,breaklinks=true,colorlinks,bookmarks=false]{hyperref}

\def\cvprPaperID{****} 
\def\confYear{CVPR 2021}

\begin{document}

\title{Supervised deep learning of elastic SRV distances on the shape space of curves}

\author{Emmanuel Hartman$^{1}$, Yashil Sukurdeep$^{2}$, Nicolas Charon$^{2}$, Eric Klassen$^{1}$, Martin Bauer$^{1}$\thanks{M. Bauer and E. Hartman  were supported by NSF grants 1953244 and 1912037. Y. Sukurdeep and N. Charon were supported by NSF grants 1945224 and 1953267. }\\\
Department of Mathematics, Florida State University, Tallahassee, USA$^{1}$\\
Department of Applied Mathematics, Johns Hopkins University, Baltimore, USA$^{2}$\\
{\tt\small \{ehartman,klassen,bauer\}@math.fsu.edu, yashil.sukurdeep@jhu.edu,charon@cis.jhu.edu}
}


\maketitle

\begin{abstract}
Motivated by applications from computer vision to bioinformatics, the field of shape analysis deals with problems where one wants to analyze geometric objects, such as curves, while ignoring actions that preserve their shape, such as translations, rotations, scalings, or reparametrizations. Mathematical tools have been developed to define notions of distances, averages, and optimal deformations for geometric objects. One such framework, which has proven to be successful in many applications, is based on the square root velocity (SRV) transform, which allows one to define a computable distance between spatial curves regardless of how they are parametrized. This paper introduces a supervised deep learning framework for the direct computation of SRV distances between curves, which usually requires an optimization over the group of reparametrizations that act on the curves. The benefits of our approach in terms of computational speed and accuracy are illustrated via several numerical experiments on both synthetic and real data. 
\end{abstract}

\section{Introduction}
\label{sec:introduction}
Motivated by applications from computer vision to bioinformatics, the field of elastic shape analysis deals with problems where one needs to analyze the variability of geometric objects ~\cite{srivastava2010shape}, \cite{amor2015action}, \cite{su2014statistical}, \cite{marron2015functional}, \cite{dai2019analyzing}. In this article, we address the computation of elastic geodesic distances between geometric curves in one and higher dimensions. By geometric curves, we mean curves modulo shape-preserving transformations, i.e., curves whose images are equal up to translations, rotations, scalings and reparametrizations. Mathematically, we model the space of all geometric curves as a quotient space of the set of absolutely continuous curves. In Section~\ref{sec:intrSRV}, we  discuss the construction of this space, and how the square root velocity (SRV) transform allows us to define the so-called SRV distance between geometric curves. For additional details, we refer interested readers to the the vast literature on elastic metrics ~\cite{srivastava2010shape},~\cite{mio2007shape},~\cite{younes1998computable},~\cite{srivastava2016functional},~\cite{bauer2014overview}.

\subsection{Background \& Related Work}
\label{ssec:background_and_related_work}
The SRV distance \cite{srivastava2010shape} quantifies dissimilarity between geometric curves, and can be used for averaging, classifying and clustering datasets of functions or curves, which are prevalent tasks in fields such as computer vision and medical imaging. However, computing SRV distances between pairs of geometric curves is a nontrivial endeavor, as it typically necessitates solving an infinite dimensional optimization problem over the set of reparametrizations that act on the curves. These optimal reparametrizations are guaranteed to exist for certain classes of curves, such as $C^1$-curves~\cite{BruverisOptimaMat}, and piecewise linear curves~\cite{lahiri2015precise}. The beauty of the result in~\cite{lahiri2015precise} lies in the fact that it not only provides an existence result, but also describes an algorithm to explicitly construct optimal reparametrizations for piecewise linear curves. Although this algorithm allows one to compute \textit{exact} SRV distances, it has a high polynomial complexity, rendering it impractical for large datasets that are typically encountered in applications.

Consequently, faster approaches have emerged to compute SRV distances in practical contexts \cite{huang2016riemannian},~\cite{srivastava2010shape}. Several such algorithms rely on dynamic programming (DP) to approximate optimal reparametrizations, operating by searching over subsets of all possible reparametrizations. The different DP-based approaches achieve runtimes of $O(n)$ to $O(n^3)$, where $n$ is the number of samples used to discretize the curves, thus providing fast over-estimates of the true SRV distance~\cite{Bernal_2016_CVPR_Workshops},~\cite{Dogan_2015_CVPR},~\cite{srivastava2010shape}. Nevertheless, DP also incurs a significant computational cost when working with very large datasets, implying that there is a need to develop increasingly efficient approaches for handling modern datasets of shapes. Towards that end, deep learning (DL) approaches have recently been introduced to estimate optimal reparametrizations for functions and curves, including supervised~\cite{lohit2019temporal},~\cite{Nunez_2020_CVPR_Workshops} and unsupervised~\cite{unsupervised},~\cite{seetharam2019structured} methods. In particular, Nunez and Joshi train a convolutional neural network (CNN) on approximate reparametrizations obtained from DP in order to estimate optimal reparametrizations for functions and curves~\cite{Nunez_2020_CVPR_Workshops}.

\subsection{Contributions}
\label{ssec:contributions}
In this paper, we propose a supervised DL framework for \textit{directly} estimating SRV distances between functions and between spatial curves, \textit{without the need to estimate optimal reparametrizations}. More specifically, we train a deep CNN to learn SRV distances, using training data consisting of pairs of discretized functions or curves, together with the SRV distance between them as labels.

As a theoretical contribution that is of interest on its own, we extend the existence results for optimal reparametrizations of~\cite{lahiri2015precise,BruverisOptimaMat} to the space of closed curves, and to the spaces of (open or closed) curves modulo rotations. These results were previously only known for open curves, and also did not consider the action of the rotation group. This in turn allows us to directly generalize the algorithm of~\cite{lahiri2015precise} for calculating exact SRV distances for open or closed curves modulo rotations.

Consequently, and in contrast to e.g.~\cite{Nunez_2020_CVPR_Workshops}, we use these \textit{exact} SRV distances as training labels for our network, rather than SRV distance over-estimates computed via DP. This reduces bias in the network's predictions. Another distinct feature of our framework is that unlike the aforementioned DP and DL approaches, we bypass the need to estimate optimal reparametrizations, instead directly estimating SRV distances. This is especially convenient for certain unsupervised learning tasks involving datasets of shapes, such as clustering applications with curves, where one only needs rapidly-computed pairwise distances rather than optimal reparametrization maps between the shapes. 

Moreover, using a neural network to predict a single number (the SRV distance) rather than a full reparametrization map is an obviously less complex learning problem, and can thus be achieved with a smaller training set, which is advantageous in the context of shape analysis where vast amounts of publicly available data are not as readily available as in e.g., the imaging sciences. Furthermore, a fundamental property of the SRV distance is its invariance to parametrizations, which we leverage to introduce a \textit{shape-preserving data augmentation} training strategy, outlined in Section~\ref{sec:DLSRV}. This training strategy allows us to augment the size of our training set while also improving the variability within training samples, which ultimately allows the trained network to produce robust, parametrization-invariant SRV distance estimates.

We are also providing an open-source version of the code for our DL framework, which is publicly available on github at \href{https://github.com/emmanuel-hartman/supervisedDL-SRVFdistances}{https://github.com/emmanuel-hartman/supervisedDL-SRVFdistances}.

To illustrate our DL framework's benefits, we show that our trained CNN's SRV distance estimates are comparable to or even more accurate than DP distances, while also being orders of magnitude faster in terms of computation time.

\section{The Shape Space and SRV Distance} 
\label{sec:intrSRV}

We begin with a brief overview of the square root velocity (SRV) framework for defining a computable elastic distance between geometric curves. In this framework, we start with parametrized curves, modelled as elements of the space of vector-valued absolutely continuous functions, denoted $\operatorname{AC}(M,\mathbb{R}^d)$. We have open curves if the parameter space $M$ is the unit interval $[0,1] \subset \R$, and closed curves if $M$ is the unit circle $S^1$. When $d=1$, we have one-dimensional curves, which we call functions.

In shape analysis for curves, one is interested in the space of all geometric curves, i.e., parametrized curves whose images are equal up to translations, rotations and reparametrizations. Note that the SRV framework, and thus our DL approach, can easily be extended to handle curves modulo scalings as well. We now briefly outline the construction of the space of all geometric curves.

To identify parametrized curves that only differ by a translation, we work with the space of absolutely continuous curves such that $c(0)=0$, denoted by $\operatorname{AC}_0(M,\mathbb{R}^d)$. We will later see that the SRV distance is naturally defined on this linear subspace of all absolutely continuous curves. Identifying curves that only differ by a reparametrization or rotation is a more delicate matter and the main source of complication in shape analysis. This is accomplished by defining the following equivalence relation for parametrized curves $c_1, c_2 \in \operatorname{AC}_0(M,\mathbb{R}^d)$. For open curves, we define $c_1\thicksim c_2$ if and only if they have the same unit speed parametrization after the application of an appropriate rotation. For closed curves, we define $c_1\thicksim c_2$ if and only if they have the same unit speed parametrization after an appropriate choice of ``starting point" and an appropriate rotation.\footnote{For functions, i.e., $d=1$, the rotation group is trivial and thus the equivalence relation reduces to factoring out the reparametrization action only.} We denote the equivalence class of a curve $c$ under this relation by $[c]$. We then define the space of all geometric curves as the quotient space:
\begin{align*}
    \mathcal{S}(M,\R^d) = \operatorname{AC}_0(M,\mathbb{R}^d) / \thicksim ~,
\end{align*}
and will refer to it as the shape space of curves, or simply as the shape space for brevity when there is no ambiguity.

We now outline how the SRV transform allows us to define a distance function on this shape space. The SRV transform is the mapping $Q: \operatorname{AC}_0(M,\mathbb{R}^d) \rightarrow L^2(M,\mathbb{R}^d)$, defined by:
\begin{align*}
    c(\cdot) \mapsto Q(c)(\cdot) := \begin{cases}\frac{c'(\cdot)}{\sqrt{|c'(\cdot)|}}& \text{if } |c'(\cdot)|>0, \\ 0 & \text{otherwise.} \end{cases}
\end{align*}
Here, $c'$ denotes the first derivative of the parametrized curve $c \in \operatorname{AC}_0(M,\mathbb{R}^d)$. This transform allows us to define the SRV distance between parametrized curves $c_1, c_2 \in \operatorname{AC}_0(M,\mathbb{R}^d)$ by pulling back the $L^2$ metric on $L^2(M,\mathbb{R}^d)$ as follows:
\begin{align*}
    d_Q(c_1,c_2)^2 
    &:= \|Q(c_1) - Q(c_2)\|_{L^2}^2 \\
    &= \int_M \Bigg| \frac{c_1'(t)}{\sqrt{|c_1'(t)|}} - \frac{c_2'(t)}{\sqrt{|c_2'(t)|}} \Bigg|^2 dt.
\end{align*}
 
It is worth noting that in the case of open curves, this distance can be interpreted as the geodesic distance induced by a Riemannian metric. For closed curves, it is only a first order approximation of a geodesic distance. The key property of this distance is its invariance under both the action of the group of rotations $\operatorname{SO}(d)$, and that of the group of diffeomorphisms of the parameter space $\Diff(M)$. The latter can be seen by a simple change of variables in the above integral. Thus, the SRV distance descends to a distance on the quotient shape space $\mathcal{S}(M,\R^d)$, given by:
\begin{align}
    \label{eq:SRVdist}
    d_{\mathcal{S}}([c_1],[c_2])
    = \inf \limits_{\substack{\gamma \in \Diff(M) \\ \operatorname{O} \in \operatorname{SO}(d)}} d_{Q}\left(c_1, \operatorname{O} \star \big( c_2 \circ \gamma \big) \right).
\end{align}

With a slight abuse of terminology, we henceforth refer to the quotient space distance, namely $d_{\mathcal{S}}(\cdot,\cdot)$, as the SRV distance. It follows that computing the SRV distance between geometric curves involves solving a joint optimization problem over the finite dimensional group $\operatorname{SO}(d)$ and the infinite dimensional reparametrization group $\Diff(M)$. The main challenge is the minimization over $\Diff(M)$, which is usually accomplished by discretizing the group into a finite dimensional approximation space, and solving the discretized problem via a dynamic programming approach. 

It is important to note that in general, the existence of a reparametrization $\gamma \in \Diff(M)$ attaining the infimum in \eqref{eq:SRVdist} is not guaranteed. However, under some additional regularity assumptions on the curves $c_1, c_2$, one can recover such existence results. In the following, we discuss the existence of optimal reparametrizations and rotations in~\eqref{eq:SRVdist}, both for the case of open curves (i.e., $M = [0,1]$) as well as closed curves (i.e., $M = S^1$).

We first introduce the semi-group of generalized reparametrizations for open curves:
\begin{equation*}
\bar \Gamma ([0,1])=\left\{\gamma \in \operatorname{AC}([0,1],[0,1]):\gamma\text{ is onto; } \gamma'\geq 0 \text{ a.e.}\right\}.    
\end{equation*} 
To introduce the analogous construction for the case of closed curves, we view $S^1$ as $\mathbb{R}/\mathbb{Z}$. We then define the shift operator on $S^1$ via:
\begin{equation*}
    S_\theta:S^1\to S^1;\qquad \lambda\mapsto \lambda+\theta.
\end{equation*}
This allows us to define the semi-group of generalized reparametrizations on $S^1$ via:
\begin{equation*}
\bar \Gamma (S^1)=\left\{S_\theta\circ \gamma^*: \theta\in S^1 \text{ and }  \gamma^*\in  \bar \Gamma ([0,1])  \right\}.    
\end{equation*} 
This allows us to formulate the following existence result, which is the main theoretical contribution of the present article:
\begin{theorem}\label{thm:existence}
Let $c_1,c_2\in \operatorname{AC}(M,\mathbb R^d)$ such that either both are of class $C^1$, or at least one of them is piecewise linear. Assume also that $c_1'$ and $c_2'$ are both nonzero a.e. on $M$. Then there exists a pair of generalized reparametrization functions $(\gamma_1,\gamma_2)\in \bar \Gamma (M) \times \bar \Gamma (M)$ and a rotation $O \in \operatorname{SO}(d)$ such that:
\begin{equation*}
 d_{\mathcal{S}}([c_1],[c_2])  = d_{Q}(c_1\circ\gamma_1,O\star(c_2\circ\gamma_2)).
\end{equation*}
\end{theorem}
Previously this result was only known for the space of open curves and did not consider the action of the rotation group, see~\cite{lahiri2015precise,BruverisOptimaMat}. The proof of Theorem~\ref{thm:existence}, which builds up on these results, is postponed to the appendix.

For piecewise linear curves $c_1, c_2 \in \operatorname{AC}_0(M,\mathbb{R}^d)$, the results of~\cite{lahiri2015precise} and Theorem~\ref{thm:existence} even lead to an algorithm that allows us to explicitly construct these optimal reparametrizations for calculating the {\bf precise} quotient space distance, see~\cite{lahiri2015precise}. This algorithm plays a fundamental role in our proposed DL framework, as we use it to calculate exact quotient SRV distances in order to generate labels for our training data, as will be outlined in the next section.

\section{Deep Learning of SRV Distances}
\label{sec:DLSRV}

While the algorithm in~\cite{lahiri2015precise} allows us to compute exact SRV distances, it is computationally expensive, making it impractical for working with large datasets of shapes. Consequently, there is a need to develop approaches that are more computationally efficient in order to calculate SRV distances. We address this need by introducing a supervised DL framework that provides fast, accurate and robust SRV distance estimates.
 
\subsection{Network Architecture} 
\begin{figure*}[hbtp!] 
    \centering
        \includegraphics[width=.8\textwidth]{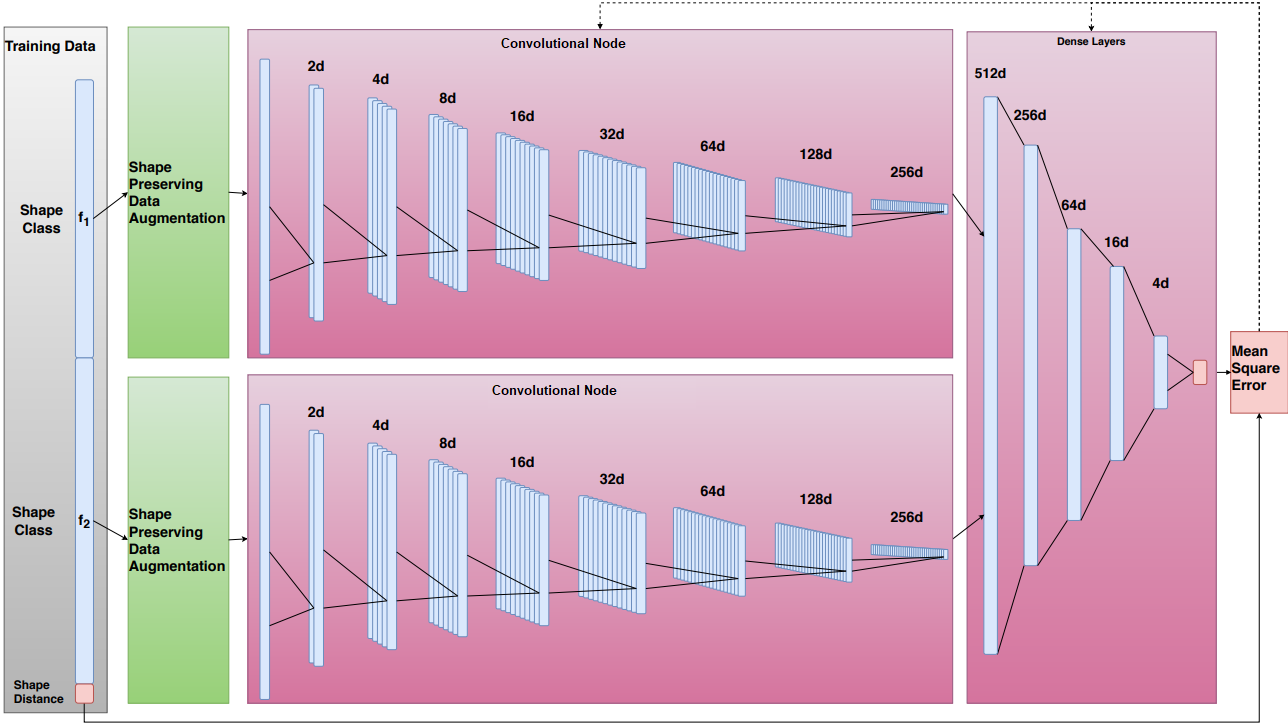}
    \caption{\small Training step and network structure diagram for Shape Preserving Data Augmentation based training: Weights contained in the red blocks are trainable and the Siamese convolutional nodes have shared weights. Specific parameter details of the network architecture can be found in Section \ref{secc:network_architecture}. The green blocks perform shape preserving data augmentation as described in Section \ref{sec:trainingmethod}.
    \label{NetStructure}}
\end{figure*}

\label{secc:network_architecture}
We train a Siamese convolutional neural network (CNN) to learn the SRV distance between geometric curves. We use training data consisting of pairs of discretized $\R^d$ valued curves, together with their SRV distance as labels. Each individual curve is sampled at $n$ vertices and represented as a flattened vector of length $n \times d$, before being fed as input to the network. Our Siamese CNN has a twin structure, consisting of two components which have identical architectures and use the same weights. To be more specific, each component of the CNN operates on an individual discretized curve, which is passed through a series of convolutional layers with kernels of size 5, followed each time by batch normalization, a rectified linear unit (ReLU) activation, and a max-pooling layer with pool size $2$. This produces two outputs, which are concatenated and passed through four dense layers whose widths are proportional to $d$, with ReLU activations being used in each dense layer. The network then outputs a single real number: the SRV distance between the two curves. We provide a schematic description of the network architecture in Fig. \ref{NetStructure}. 

\subsection{Training Method} 
\label{sec:trainingmethod}

\begin{figure}[htbp!]
\begin{center}
    \includegraphics[width=0.27\textwidth,angle=90]{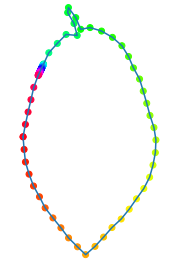}
    \includegraphics[width=.1\textwidth]{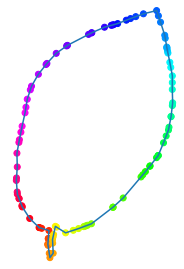}
    \includegraphics[width=.15\textwidth]{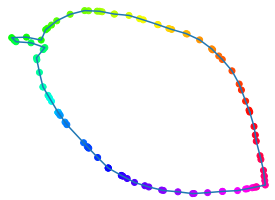}
    \includegraphics[width=.12\textwidth]{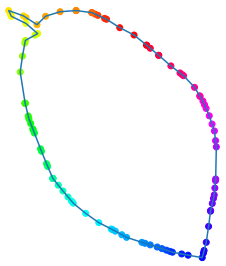}
    \end{center}
        \caption{\small Example of Shape Preserving Data Augmentation: The top curve is an example of a parameterization of a curve from the Swedish Leaf II dataset, see Section~\ref{sec:numerical_experiments} for a description of this dataset. The three curves on the bottom represent parameterizations and rotations of this curve as produced by the shape preserving data augmentation described in Section \ref{sec:trainingmethod}.}
        \label{fig:SPDA}
\end{figure}

We create training and testing sets for our network by randomly generating pairs of functions or curves, or by picking them from an existing dataset, and labelling them with their SRV distance. We use exact distances computed with the algorithm of \cite{lahiri2015precise} as labels for functions and 2D curves, but due to this algorithm's high complexity, we instead use DP distances as labels for 3D curves.

Computing these SRV distance labels using the exact algorithm or DP may be very time consuming, which could limit the size of our training set in practice. Thankfully, from a base training set, one can easily generate more training samples at no extra cost by applying shape-preserving transformations, such as resampling and rotations, to both curves. Indeed, the quotient SRV distance is invariant to reparametrizations (i.e., to resampling in the discrete situation) and to rotations, implying that the distance between resampled and/or rotated curves remains unchanged and need not be recomputed. This data augmentation strategy allows the network to see a wider variety of sampling patterns and rotations for the same curve during training, which helps it to learn and predict distances that are truly invariant to reparametrizations and rotations. Moreover, since resampling and rotating curves is computationally inexpensive, this procedure can be performed at each iteration of the training step, without incurring any additional storage for new training samples. We empirically observed that this \textit{shape-preserving data augmentation}-based training method reduced overfitting in the distance learned by the network.

The training itself is performed using an Adam optimization procedure \cite{kingma2014adam}. We observed relatively fast convergence in all cases, with convergence curves shown in Fig.~\ref{fig:convergence_curves}. We refer readers to the code documentation on \href{https://github.com/emmanuel-hartman/supervisedDL-SRVFdistances}{https://github.com/emmanuel-hartman/supervisedDL-SRVFdistances} for further training details, including information on the exact training parameters such as the batchsize for each epoch of training, and the parameters involved in the shape-preserving data augmentation training step.

\begin{figure}[h!]
\begin{center}
    \includegraphics[width=0.23\textwidth]{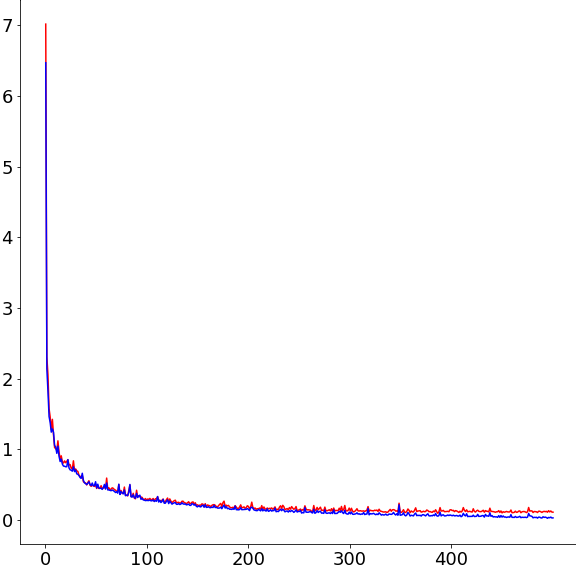}
    \includegraphics[width=0.23\textwidth]{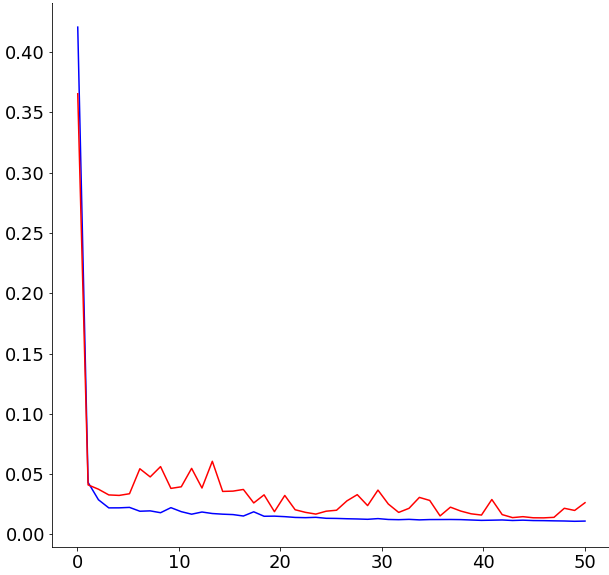}
    \end{center}
        \caption{\small On both figures, the $x$-axis represents epochs, and on the $y$-axis, we plot the mean squared error of the network on the training data (blue), as well as on unseen testing data (red). Convergence curves for network trained on open, real-valued functions discretized at 90 points from our Synthetic I data set, trained for 500 epochs (left figure). Convergence curves for network trained on closed, 2-dimensional curves discretized at 100 points from the Kimia dataset, trained for 50 epochs (right figure). Descriptions of the datasets are given in Section~\ref{sec:numerical_experiments}.}
        \label{fig:convergence_curves}
\end{figure}

\section{Numerical Experiments}
\label{sec:numerical_experiments}

We now present empirical results demonstrating the performance of our DL approach for estimating SRV distances on real-valued functions, and on curves in $\R^2$ and $\R^3$. As we shall see, the experiments show that when compared to DP, our approach produces SRV distance estimates at a significantly lower numerical cost, while being comparable, and sometimes superior, in terms of accuracy.

\subsection{Computation Method}
To compute SRV distances using the exact algorithm and DP, we used Martins Bruveris' package\footnote{\href{https://github.com/martinsbruveris/libsrvf}{https://github.com/martinsbruveris/libsrvf}}, which builds on the DP code of FSU's Statistical Shape Analysis and Modeling Group. Our network was implemented on TensorFlow. All computation times using the different algorithms were recorded on an Intel Xeon X5650 2.66 GHz CPU with a Gigabyte GeForce GTX 1060 1582 MHz GPU. A comparision of the computation times for the various algorithms can be found in Table~\ref{timing}. One can clearly see that the trained network is several orders of magnitude faster than both DP and the exact algorithm.
\begin{table}[h]
\centering
    \begin{minipage}{.45\textwidth}
        \centering
 \begin{tabular}{||c |c c c c||}
     \hline 
       & Exact & DP &  DL (CPU)& DL (GPU) \\ [0.5ex] 
     \hline\hline
        1D &  $2 \times 10^3$ & $5 \times 10^2$ & $5 \times 10^{-2}$ & $2 \times 10^{-2}$ \\ 
        2D &  $2 \times 10^5$ & $8 \times 10^2$ & $2 \times 10^{-1}$ & $3 \times 10^{-2}$  \\
     \hline
    \end{tabular}
    \vspace{.1cm}
    \caption{{\footnotesize Computation time for one SRV distance using several different algorithms, in milliseconds.}}
    \label{timing}
    \end{minipage}
\end{table}

\subsection{Evaluation Method} 
\textbf{For functions and 2D curves:} To evaluate the trained network's accuracy for functions and 2D curves, we use the mean relative error (MRE) between its output and the true SRV distances on a test set, computed via the exact algorithm. As a secondary measure of estimation quality, we use the Pearson correlation coefficient 
\begin{equation*}
    \rho_{y \widehat{y}} = \frac{\sum\limits_{i=1}^N(y_i-\overline y)(\widehat{y_i}-\overline{\widehat{y}})}{\sqrt{\sum\limits_{i=1}^N(y_i-\overline y)^2}\sqrt{\sum\limits_{i=1}^N(\widehat{y_i}-\overline{\widehat{y}})^2}}
\end{equation*}
between the network's output and exact distances on a test set of functions or 2D curves. Here $N$ is the number of training samples, $\{y_i\}_{i=1}^N$ are the exact distances, $\{\widehat{y_i}\}_{i=1}^N$ are the outputs of the network, with $\overline{y} = \frac{1}{N}\sum_{i=1}^N y_i$ and $\overline{\widehat{y}} = \frac{1}{N}\sum_{i=1}^N \widehat{y_i}$ being their respective sample means. A low MRE and a strong positive correlation coefficient indicates a good performance of the network.

\textbf{For 3D curves:} In addition to the results for functions and planar curves, we present preliminary results for curves in $\mathbb R^3$. As the computational complexity of the exact algorithm is orders of magnitude higher for curves in $\mathbb R^3$ when compared to the case of functions and curves in $\R^2$, we only label 3D curves with DP distances. Consequently we can only evaluate the CNN's performance for 3D curves via the correlation coefficient between its output and DP distances.

To avoid bias in our results, {\bf elements of the test set are never contained in the training set}, which is used solely for the purpose of calibrating the network. Furthermore, to assess our network's generalization capabilities, we make sure that {\bf the trained network is tested on data that is significantly different compared to the data used for training}, c.f. Figures \ref{fig:example_functions}~and~\ref{fig:ex_2dcurves}.

\subsection{Experiments with functions}
\textbf{Datasets:} First, we tested our network's ability to predict SRV distances for functions, using both synthetic and real data. The synthetic data was created by generating functions sampled at $90$ evenly spaced points on the unit interval with random arc length. We note that the shape class of a function modulo reparametrizations is entirely determined by its local maxima and minima, as it is determined by its constant speed parameterization which is a linear interpolation between the local maxima and minima. Thus, the synthetic function data is generated by drawing the number of extrema for our function from a $\mathcal{N}(\mu, \sigma)$ distribution, randomly assigning values for these extrema, and then randomly choosing a function with $90$ breakpoints from the shape class determined by the generated extrema. We created two different synthetic datasets: the first one with parameters $(\mu, \sigma) = (18, 6)$, and the second with $(\mu, \sigma) = (30, 10)$. These datasets, dubbed Synthetic I and Synthetic II respectively, each contain 100,000 pairs of functions labelled with their exact distances, partitioned into 99,000 training cases and 1,000 testing cases. For the real dataset, we use CPC Global Unified Precipitation data from the NOAA/OAR/ESRL PSL, Boulder, Colorado, USA\footnote{\href{https://psl.noaa.gov/}{https://psl.noaa.gov/}}, from which we extracted 90 days of precipitation data across several locations and years. We randomly selected 400 samples from this database, computed exact pairwise distances, and partitioned them into a set of 89700 distances for training, and 9,900 for testing. See Fig.~\ref{fig:example_functions} for examples from the different datasets.

\begin{figure}[h!]
    \centering
        \includegraphics[width=.22\textwidth]{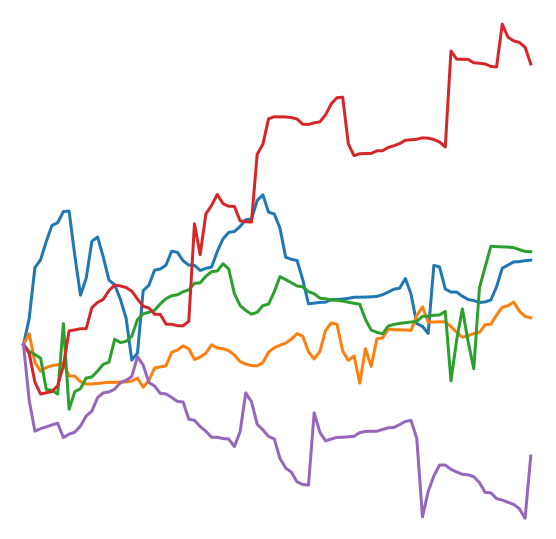}
        \includegraphics[width=.22\textwidth]{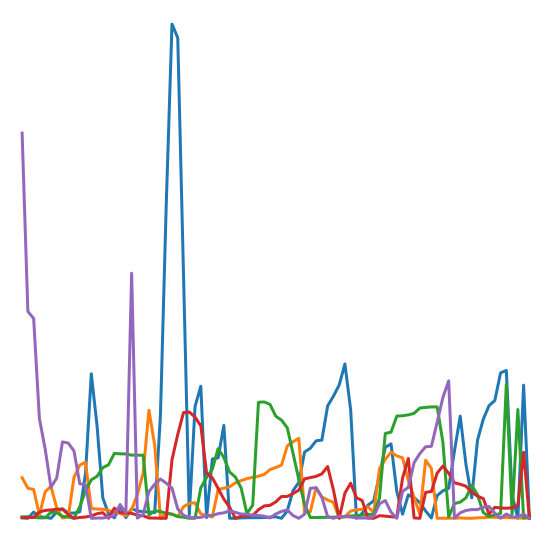}
        \includegraphics[width=.22\textwidth]{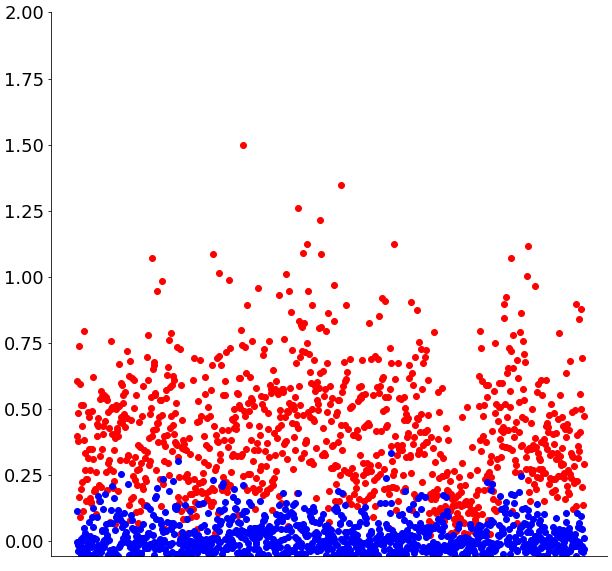} 
        \includegraphics[width=.22\textwidth]{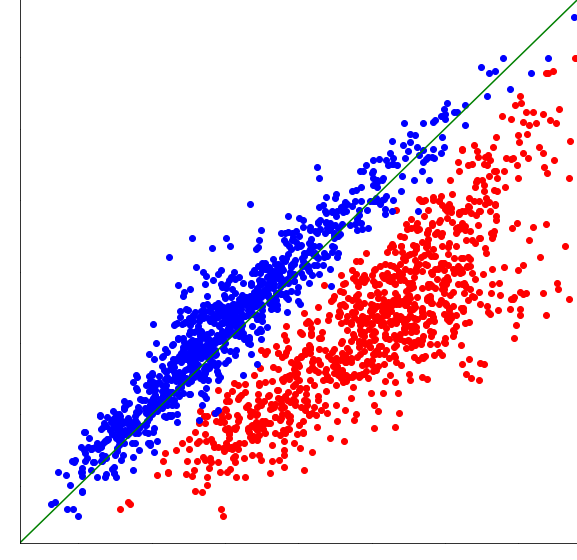} 
    \caption{\small  Five examples from Synthetic I (top-left) and the CPC Precipitation dataset (top-right). Third and fourth figure: 
    Comparison of DP (red) and Trained Network (blue). Scatter plot of relative errors for 1000 testing cases from the CPC Precipitation dataset, using a network trained on Synthetic I (bottom-left). Corresponding correlation plot for both methods, with exact distances on the $y$-axis, and estimated distances on the $x$-axis, and the line $y=x$ in green (bottom-right).
    \label{fig:example_functions}} 
\end{figure}

\textbf{Results:} First, we highlight the difference in performance between our trained network and DP, see Table~\ref{tab:DL_vs_DP_1} and Fig.~\ref{fig:example_functions}. When trained and validated on the same type of data, the network significantly outperforms DP across all three datasets, both in terms of the MRE and correlation coefficient with respect to the exact distances, see Table~\ref{tab:DL_vs_DP_1}.

To demonstrate our network's generalization capabilities, we trained it on one type of data and tested it on a different dataset, e.g., we trained on synthetic data but tested on CPC precipitation data. While this leads to a slight increase in prediction error, the network still outperforms DP on both measures by a large margin, see Table~\ref{tab:gen_results_1}.

\begin{table}[h!]
\centering
\begin{tabular}{||c|c c|c c||} 
 \hline
 \multirow{2}{*}{Dataset}& \multicolumn{2}{c|}{MRE}& \multicolumn{2}{c||}{Corr.}\\[0.5ex] 
 \cline{2-5} 
          &  DP      &  DL     &    DP    & DL      \\[0.5ex] 
 \hline
 Synthetic I & 0.44551 & 0.04346 & 0.84307 & 0.96770 \\ 
 Synthetic II & 0.45949 & 0.03806 & 0.87123 & 0.97053 \\ 
 CPC Precip. & 0.45853 & 0.03722 & 0.85452 & 0.96090 \\ 
 \hline
 \hline
\end{tabular}
\vspace{.1cm}
\caption{\footnotesize Comparison between DP and DL\label{tab:DL_vs_DP_1}}
\end{table}

\begin{table}[h!]
\centering
    \begin{minipage}{.45\textwidth}
        \centering
        
\begin{tabular}{||c c c c||} 
 \hline 
Training Set & Testing Set & MRE & Corr. \\ [0.5ex] 
 \hline
Synthetic I &  Synthetic II & 0.05520 & 0.96782 \\ 
Synthetic I & CPC Precip. & 0.08088 & 0.95264 \\ 
 \hline
Synthetic II & Synthetic I & 0.05206 & 0.96110 \\ 
 Synthetic II & CPC Precip. & 0.07093 & 0.94888 \\ 
 \hline 
\end{tabular}
\vspace{.1cm}
    \caption{\footnotesize Generalization results across several testing sets  \label{tab:gen_results_1}}
        \centering
    \end{minipage}
\end{table}

\subsection{Experiments with curves in $\R^2$}
\label{ssec:2dcurve_experiments}
\textbf{Datasets:} We used data from the MPEG-7\footnote{\href{https://dabi.temple.edu/external/shape/MPEG7/dataset.html}{https://dabi.temple.edu/external/shape/MPEG7/dataset.html}} and Swedish leaf datasets\footnote{\href{https://www.cvl.isy.liu.se/en/research/datasets/swedish-leaf/}{https://www.cvl.isy.liu.se/en/research/datasets/swedish-leaf/}}, which contain images of objects whose boundaries were extracted and treated as 2D curves, discretized with $100$ points, see Fig.~\ref{fig:ex_2dcurves}. To extract discretized boundary curves from these datasets, we binarized each image via Otsu's algorithm, then extracted vertices on the boundary using the Moore-Neighbor tracing algorithm, before downsampling to $100$ points.

We trained the network on 229162 distinct pairs of curves labelled with exact distances from the MPEG-7 dataset, which contains a diverse array of 2D shapes from many different shape classes, see Fig.~\ref{fig:ex_2dcurves}. We tested the network on two versions of the Swedish leaf dataset, called Swedish Leaf I and II respectively. Swedish Leaf I contains curves with arc length parametrizations, i.e., discretized with $n$ points that are uniformly distributed across the curve. Swedish Leaf II contains ``adversarial parametrizations", i.e., curves with $n$ points that are far from uniformly distributed across the curve, with many points concentrated on a small portion of the curve, see Fig.~\ref{fig:swedish_leaves_stdvsreparam}.

\begin{figure}[h!]
\begin{center}
    \includegraphics[height=0.2\textwidth]{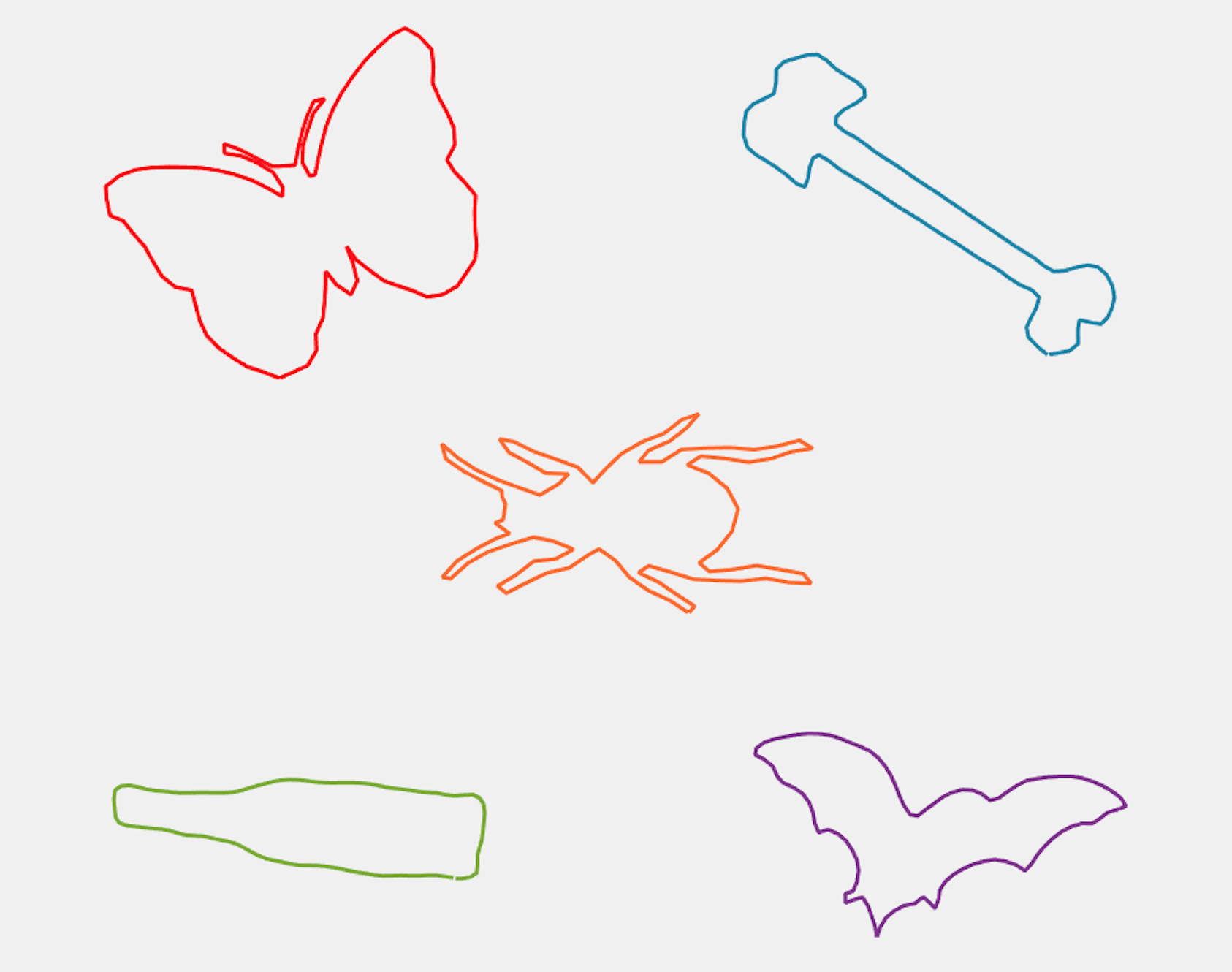}
    \includegraphics[height=0.2\textwidth]{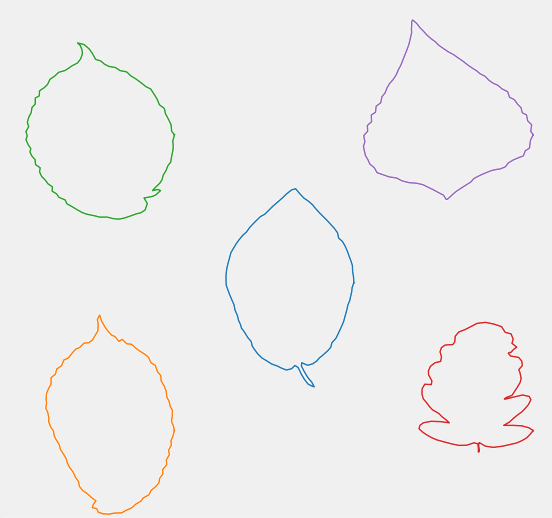}
\end{center}
\caption{Five examples from the MPEG-7 dataset (left). Five examples from the Swedish leaf dataset (right).}
\label{fig:ex_2dcurves}
\end{figure}

\begin{figure}[h!]
\begin{center}
    \includegraphics[width=0.23\textwidth]{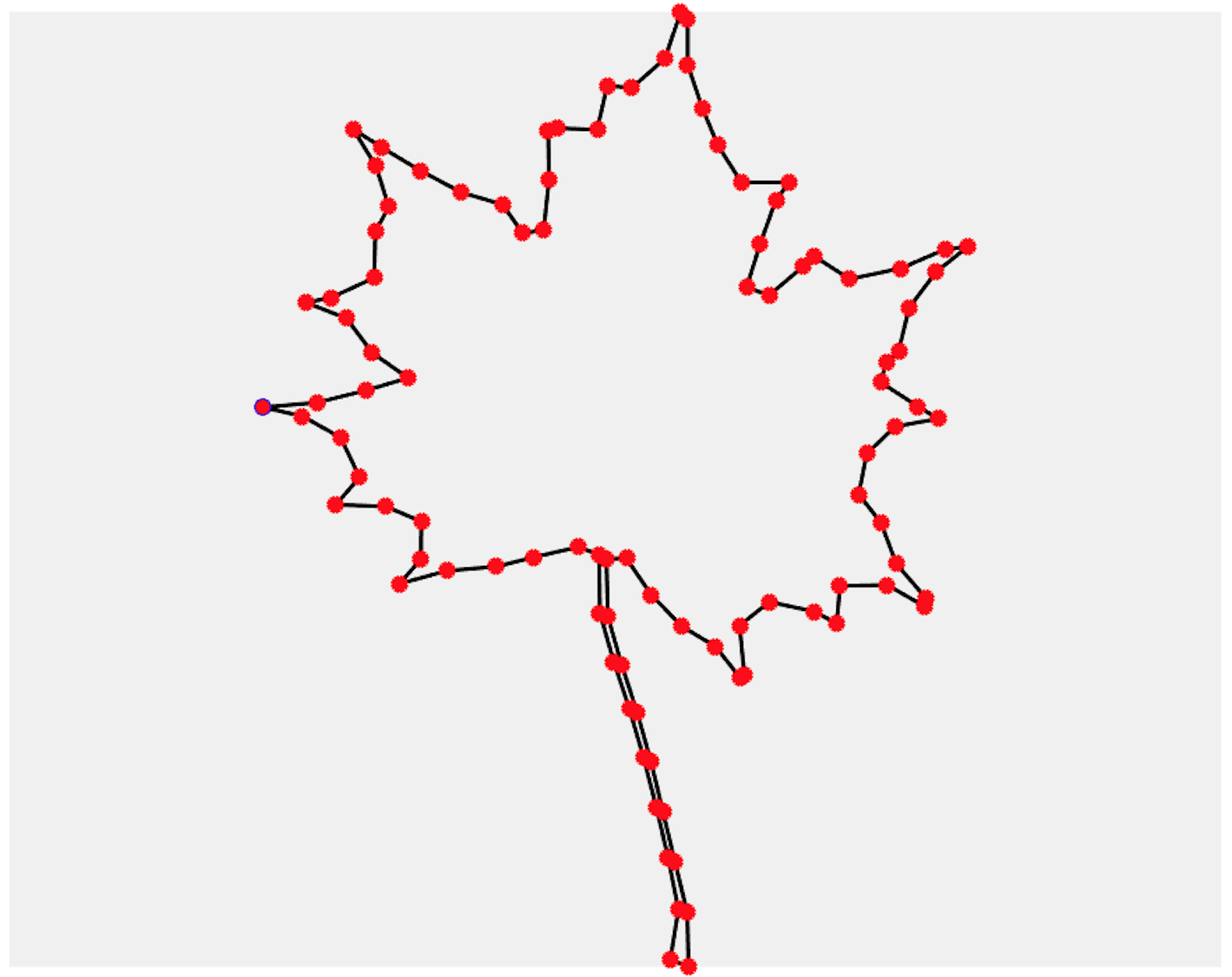} 
    \includegraphics[width=0.23\textwidth]{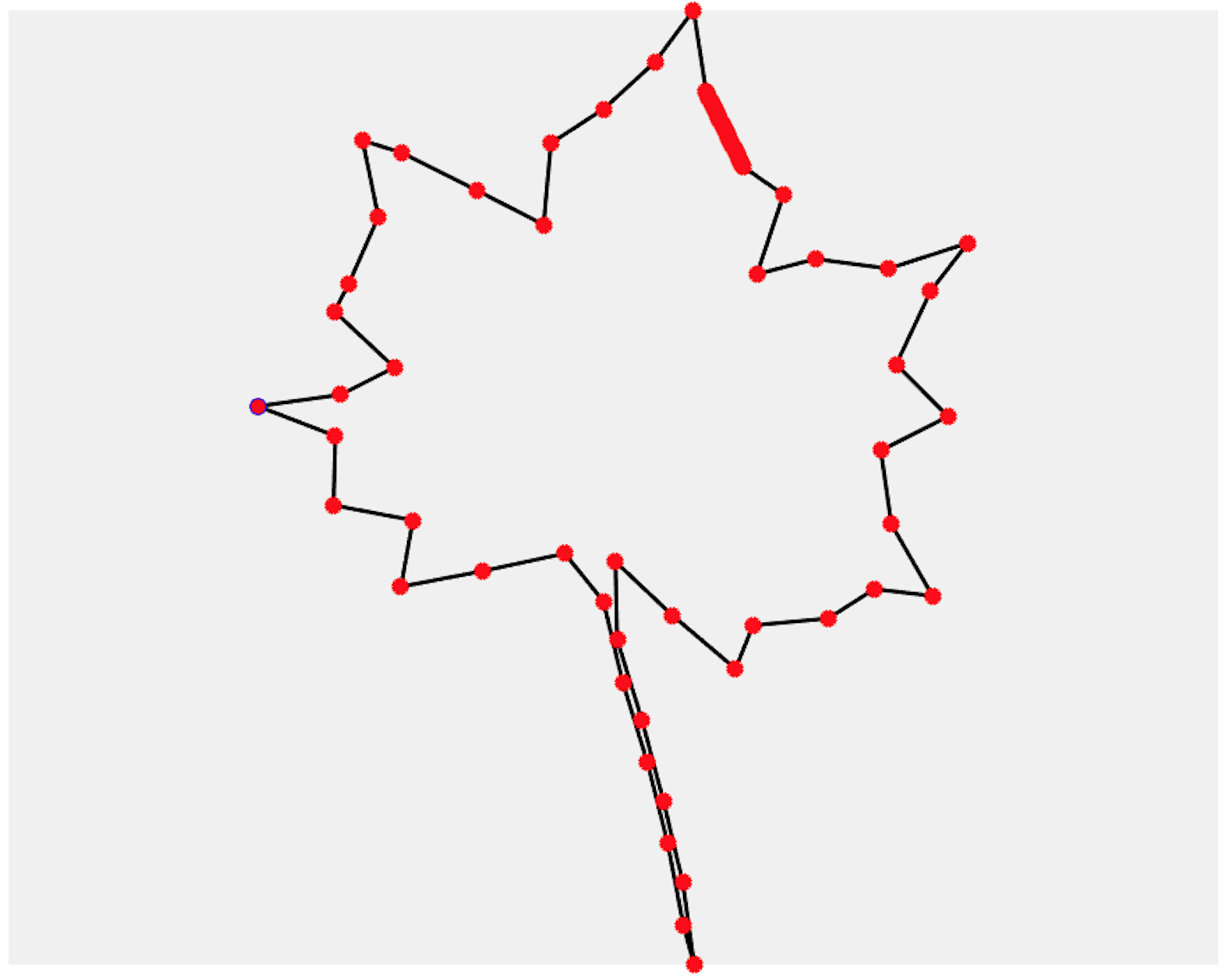}
\end{center}
\caption{Example of a curve from the Swedish Leaf I dataset, where curves have arc length parametrization (left), and from the Swedish Leaf II dataset with adversarial parametrizations (right).}
\label{fig:swedish_leaves_stdvsreparam}
\end{figure}

\textbf{Results:}  Due to the higher dimensionality and complexity of the data, the network's performance drops compared to the case of functions. However, we still obtain a high correlation coefficient with the exact distance across both datasets, namely 0.924 for Swedish Leaf I, and 0.917 for Swedish Leaf II. Meanwhile, the corresponding correlation coefficient for DP distances is 0.996 for Swedish Leaf I, but drops significantly to 0.899 for Swedish Leaf II. These observations show comparable performance between our DL framework and DP, with DP being more accurate for curves that are already well-aligned (e.g., for those in Swedish Leaf I), and DL being superior in terms of accuracy for data requiring larger reparametrizations (e.g., for curves in Swedish Leaf II). 

As yet another proof of concept for our DL approach, we perform an unsupervised clustering experiment using $40$ curves taken from the Swedish Leaf I dataset. These curves are evenly distributed across four categories of leaves. To cluster the curves, we compute all pairwise SRV distances using both our DL framework and the exact algorithm, and apply classical multidimensional scaling (CMDS) to the resulting pairwise distance matrices in order to obtain a 2D projection of the dataset, see Fig.~\ref{fig:Clusters}. While the resulting 2D visualizations are slightly different, the clusters produced are comparable.

\begin{figure}[h!]
    \begin{center}
        \includegraphics[width=0.23\textwidth]{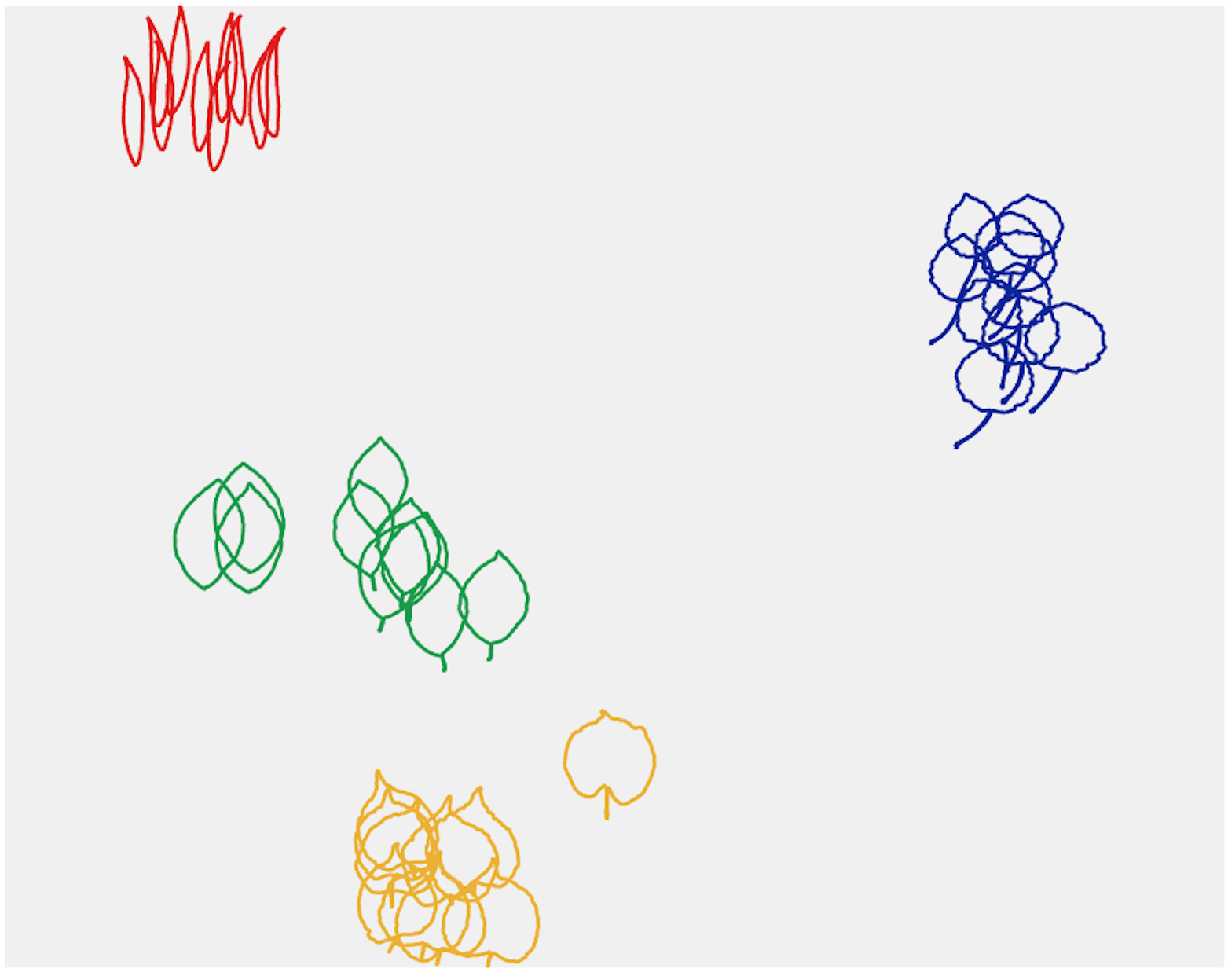}
        \includegraphics[width=0.23\textwidth]{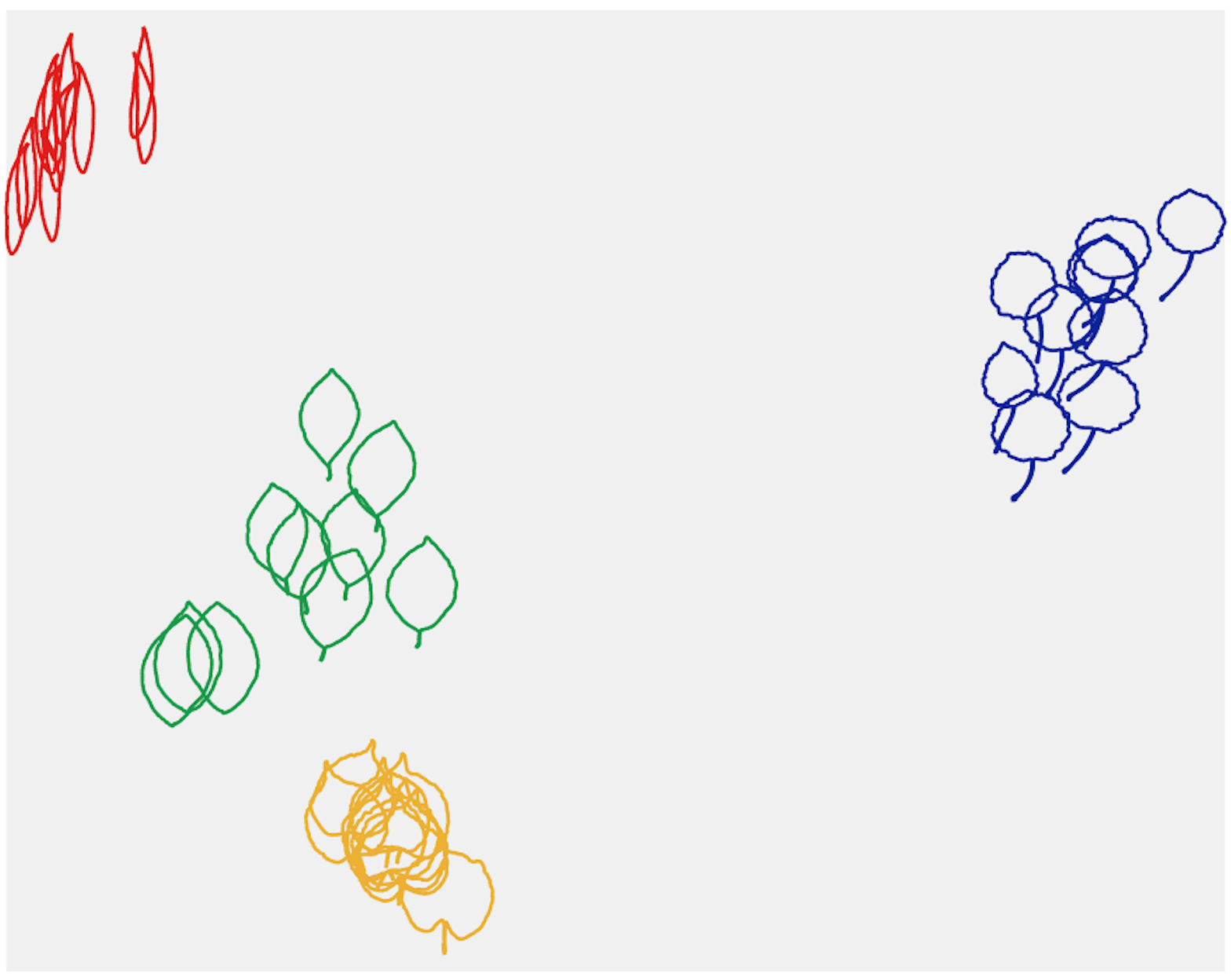}
    \end{center}
    \caption{\small CMDS clusters of 40 curves selected from the Swedish leaf dataset using exact distances (left) and DL distances (right). 
    \label{fig:Clusters}}
\end{figure}

\subsection{Preliminary experiments for curves in $\R^3$}
\textbf{Datasets:} Finally we present preliminary results for curves in $\R^3$. We use two distinct datasets of open 3D curves for the experiments: hurricane paths from the National Hurricane Center Data Archive\footnote{\href{https://www.nhc.noaa.gov/data/}{https://www.nhc.noaa.gov/data/}} and plant roots\footnote{\href{https://github.com/RSA-benchmarks/collaborative-comparison}{https://github.com/RSA-benchmarks/collaborative-comparison}} from which we only keep the taproots (i.e. main stem of the roots). All curves were discretized with $100$ points. We trained the network on 284622 distinct pairs of hurricane paths, labelled with distances computed with the DP algorithm. We validated the network on either a different testing set of hurricane paths, or on the dataset of taproots. The reason for choosing DP distances instead of exact distances for training the network is the high computational cost of calculating exact distances in the situation of 3D curves. 


\textbf{Results:} The correlation coefficient between the predicted DL and the DP distances on the test set of hurricane paths is 0.977, but drops significantly to 0.823 when tested on taproots. This drop-off in prediction quality can be most likely explained by the lack of sufficiently diverse samples in the training set, which limits the network's generalization capabilities. We expect that enriching the training set with 3D curves displaying more varied geometries will help to improve the network's performance. However, due to the scarcity of publicly available datasets of 3D curves, we leave it as future work to build a better training set for 3D curves and invest the required computational resources to train our network using SRV distance labels computed via the exact algorithm, as was done for functions and 2D curves.

\section{Conclusion}
\label{sec:conclusion}
We have introduced a supervised DL framework to compute SRV distances for curves in $\R^d$. The main advantage of our approach is that our trained CNN provides fast and accurate estimates of the SRV distance between pairs of geometric curves, without the need to find optimal reparametrizations. Moreover, we exploited the invariance of the SRV distance to shape-preserving actions in order to propose a shape-preserving data-augmentation based training strategy, which is a flexible and efficient procedure for creating and augmenting our training set. Empirical observations show that this training strategy allows our network to estimate SRV distances which are truly invariant to rotations and reparametrizations, while also reducing overfitting. Moreover, our experiments show that when compared to DP, our approach produces SRV distance estimates at a significantly lower numerical cost, while also being comparable, and sometimes superior, in terms of accuracy.

\subsection*{Acknowledgements}
We want to thank Martins Bruveris for helpful discussions during the preparation of this manuscript.

\section*{Appendix}
\begin{proof}[Proof of Theorem~\ref{thm:existence}]
Not taking into account the action of the group of rotations and considering only open curves, this result was shown in ~\cite{lahiri2015precise}, assuming that one of the curves is piecewise linear, and in \cite{BruverisOptimaMat} under the assumptions that both curves are of class $C^1$.  

In the case where $M=S^1$, by the definition of $\bar \Gamma(S^1)$, the existence of a pair of optimal $(\gamma_1,\gamma_2)\in \bar \Gamma(S^1) \times \bar \Gamma(S^1)$ is equivalent to the existence of an optimal $\tau\in S^1$, an optimal $O\in \operatorname{SO}(d)$ and a pair of optimal $\gamma^*_1,\gamma^*_2\in \overline{\Gamma}([0,1])$. 
Consider the function $F:S^1 \times \operatorname{SO}(d) \to \R$ given by 
\begin{equation*}
 F(\lambda,O)=\inf\limits_{\gamma^*_1,\gamma^*_2\in \overline{\Gamma}([0,1])} d_{Q}(c_1\circ\gamma^*_1,O\star(c_2\circ S_\lambda\circ\gamma^*_2)).    
\end{equation*}
We will first show that $F$ is continuous. Let $\tau\in S^1$, $O\in \operatorname{SO}(d)$ and $\epsilon>0$.
Since $C(S^1,\R^d)$ is dense in $L^2(S^1,\R^d)$, let $g\in C(S^1,\R^d)$ such that $ ||Q(c_2)-g||_{L^2}<\epsilon/4$. As $g$ is continuous on a compact domain it follows, by the Heine-Cantor theorem, that it is uniformly continuous. Thus, there exists $\delta>0$ such that for each $\theta\in S^1$ and each $\lambda\in S^1$ such that $|\lambda|\leq\delta,$  we have $|g(\theta)-g(S_\lambda(\theta))|<\epsilon/4.$ Thus, for each $\lambda$ such that  $|\lambda|\leq\delta$, we have 
\begin{align*}
    ||g-g\circ S_\lambda||^2_{L^2}&=\int_{S^1}|g(\theta)-g(S_\lambda(\theta))|^2d\theta\\&<\int_{S^1}(\epsilon/4)^2 d\theta=(\epsilon/4)^2.
\end{align*}
Pick this $\delta$ and let $\lambda\in S^1$ such that $|\tau-\lambda|<\delta.$ Thus, \begin{align*}
    ||g\circ S_\tau-g\circ S_\lambda||_{L^2}=||g-g\circ S_{\lambda-\tau}||_{L^2}<\epsilon/4.
\end{align*}
By a change of variable argument, we can show for any $g_1,g_2\in C(S^1,\R)$ and $\theta\in S^1$, we have 
\begin{equation*}
    ||g_1\circ S_\theta-g_2\circ S_\theta||_{L^2}=||g_1-g_2||_{L^2}.
\end{equation*} 
Furthermore, it is easy to show that for any $c\in \operatorname{AC}(S^1,\mathbb{R}^d)$ and any $\theta\in S^1$, we have $Q(c\circ S_\theta)=Q(c)\circ S_\theta$. On the other hand, the action of the rotation group on curves induces a corresponding action on their SRV transform which we write for any $O \in \operatorname{SO}(d)$ as $Q(O\star c) = O\cdot Q(c)$, where we have specifically that $O\cdot Q(c)(\cdot) = \frac{1}{\sqrt{|c'(\cdot)|}} O c'(\cdot)$. Note that this action on SRV transforms is by isometry for $\|\cdot\|_{L^2}$. 

Now, for any $(\lambda, O') \in S^1 \times \operatorname{SO}(d)$ with $|\tau-\lambda|<\delta$ and $||O-O'|| < \epsilon/(4||Q(c_2)||_{L^2})$ (for the operator norm on matrices), we can write
\begin{align*}
|F(\lambda&,O)-F(\tau,O')|\\
\leq & ||O\cdot Q(c_2\circ S_\lambda)-O'\cdot Q(c_2\circ S_\tau)||_{L^2} \\
=&||O\cdot Q(c_2)\circ S_\lambda-O' \cdot Q(c_2)\circ S_\tau||_{L^2}\\
\leq& ||O \cdot Q(c_2)\circ S_\lambda - O' \cdot Q(c_2)\circ S_\lambda ||_{L^2}\\ 
&+ ||O' \cdot Q(c_2)\circ S_\lambda - O'\cdot Q(c_2)\circ S_\tau||_{L^2}.
\end{align*}
For the first term on the right hand side, we can see that 
\begin{equation*}
  ||O \cdot Q(c_2)\circ S_\lambda - O' \cdot Q(c_2)\circ S_\lambda ||_{L^2} \leq ||O-O'||.||Q(c_2)\circ S_\lambda||_{L^2},
\end{equation*}
and since $||Q(c_2)\circ S_\lambda||_{L^2} =||Q(c_2)||_{L^2}$, we can bound this term by $\epsilon/4$. On the other hand, we have
\begin{align*}
&||O' \cdot Q(c_2)\circ S_\lambda - O'\cdot Q(c_2)\circ S_\tau||_{L^2} \\ 
&=||Q(c_2)\circ S_\lambda - Q(c_2)\circ S_\tau||_{L^2}\\
&\leq ||Q(c_2)\circ S_\lambda-g\circ S_\lambda||_{L^2} +||g\circ S_\lambda-g\circ S_\tau||_{L^2} \\
&\phantom{||}+||g\circ S_\tau-Q(c_2)\circ S_\tau||_{L^2} \\
&=||Q(c_2)-g||_{L^2}+||g\circ S_\lambda-g\circ S_\tau||_{L^2} +||g-Q(c_2)||_{L^2} \\
&<\epsilon/4+\epsilon/4+\epsilon/4=3\epsilon/4,
\end{align*}
which finally leads to $|F(\lambda,O)-F(\tau,O')| <\epsilon$. Now, since $F$ is continuous on the compact set $S^1 \times \operatorname{SO}(d)$, there exists an optimal $\tau\in S^1$ and an optimal $O\in \operatorname{SO}(d)$ such that $F(\tau,O)=\inf_{S^1 \times \operatorname{SO}(d)} F$. Note that the curves $c_1$ and $O\star (c_2\circ S_\tau)$ belong to $\operatorname{AC}(S^1,\mathbb R^d)$ and thus in particular to $\operatorname{AC}([0,1],\mathbb R^d)$, and that by assumption, they are either both of class $C^1$ or one of them is piecewise linear. By the results of \cite{BruverisOptimaMat,lahiri2015precise}, there exist optimal $\gamma^*_1,\gamma^*_2\in \overline{\Gamma}([0,1])$ such that 
\begin{equation*}
 d_{\mathcal{S}}([c_1],[c_2])= d_{Q}(c_1\circ\gamma^*_1,O\star(c_2\circ S_\tau\circ\gamma^*_2)), 
\end{equation*}
which concludes the proof for the case of closed curves. The proof for open curves modulo reparametrizations and rotations can be done exactly as above, by considering a function $\tilde F$ that only depends on rotations.   
\end{proof}



\bibliographystyle{plain}

\begin{thebibliography}{10}

\bibitem{amor2015action}
Boulbaba~Ben Amor, Jingyong Su, and Anuj Srivastava.
\newblock Action recognition using rate-invariant analysis of skeletal shape
  trajectories.
\newblock {\em IEEE transactions on pattern analysis and machine intelligence},
  38(1):1--13, 2015.

\bibitem{bauer2014overview}
Martin Bauer, Martins Bruveris, and Peter~W Michor.
\newblock Overview of the geometries of shape spaces and diffeomorphism groups.
\newblock {\em Journal of Mathematical Imaging and Vision}, 50(1-2):60--97,
  2014.

\bibitem{Bernal_2016_CVPR_Workshops}
Javier Bernal, Gunay Dogan, and Charles~R Hagwood.
\newblock {Fast Dynamic Programming for Elastic Registration of Curves}.
\newblock In {\em The IEEE Conference on Computer Vision and Pattern
  Recognition (CVPR) Workshops}, June 2016.

\bibitem{BruverisOptimaMat}
Martins Bruveris.
\newblock {Optimal Reparametrizations in the Square Root Velocity Framework}.
\newblock {\em SIAM Journal on Mathematical Analysis}, 48, 07 2015.

\bibitem{dai2019analyzing}
Mengyu Dai, Zhengwu Zhang, and Anuj Srivastava.
\newblock Analyzing dynamical brain functional connectivity as trajectories on
  space of covariance matrices.
\newblock {\em IEEE transactions on medical imaging}, 39(3):611--620, 2019.

\bibitem{Dogan_2015_CVPR}
Gunay Dogan, Javier Bernal, and Charles~R Hagwood.
\newblock {A Fast Algorithm for Elastic Shape Distances Between Closed Planar
  Curves}.
\newblock In {\em The IEEE Conference on Computer Vision and Pattern
  Recognition (CVPR)}, June 2015.

\bibitem{DL_SRV}
E.~Hartman, Y.~Sukurdeep, E.~Klassen, N.~Charon, and M.~Bauer.
\newblock {S}upervised{DL-SRVF}distances:
  \url{https://github.com/emmanuel-hartman/supervisedDL-SRVFdistances}, 2021.

\bibitem{huang2016riemannian}
Wen Huang, Kyle~A Gallivan, Anuj Srivastava, and Pierre-Antoine Absil.
\newblock Riemannian optimization for registration of curves in elastic shape
  analysis.
\newblock {\em Journal of Mathematical Imaging and Vision}, 54(3):320--343,
  2016.

\bibitem{kingma2014adam}
Diederik~P Kingma and Jimmy Ba.
\newblock {Adam: A Method for Stochastic Optimization}, 2014.

\bibitem{unsupervised}
K.~{Koneripalli}, S.~{Lohit}, R.~{Anirudh}, and P.~{Turaga}.
\newblock {Rate-Invariant Autoencoding of Time-Series}.
\newblock In {\em ICASSP 2020 - 2020 IEEE International Conference on
  Acoustics, Speech and Signal Processing (ICASSP)}, pages 3732--3736, 2020.

\bibitem{lahiri2015precise}
Sayani Lahiri, Daniel Robinson, and Eric Klassen.
\newblock {Precise Matching of PL Curves in $R^N$ in the Square Root Velocity
  Framework}.
\newblock {\em Geometry, Imaging and Computing}, 2, 01 2015.

\bibitem{lohit2019temporal}
Suhas Lohit, Qiao Wang, and Pavan Turaga.
\newblock Temporal transformer networks: Joint learning of invariant and
  discriminative time warping.
\newblock In {\em Proceedings of the IEEE Conference on Computer Vision and
  Pattern Recognition}, pages 12426--12435, 2019.

\bibitem{marron2015functional}
James~Stephen Marron, James~O Ramsay, Laura~M Sangalli, and Anuj Srivastava.
\newblock Functional data analysis of amplitude and phase variation.
\newblock {\em Statistical Science}, pages 468--484, 2015.

\bibitem{mio2007shape}
Washington Mio, Anuj Srivastava, and Shantanu Joshi.
\newblock On shape of plane elastic curves.
\newblock {\em International Journal of Computer Vision}, 73(3):307--324, 2007.

\bibitem{Nunez_2020_CVPR_Workshops}
Elvis Nunez and Shantanu~H Joshi.
\newblock Deep learning of warping functions for shape analysis.
\newblock In {\em Proceedings of the IEEE/CVF Conference on Computer Vision and
  Pattern Recognition (CVPR) Workshops}, June 2020.

\bibitem{seetharam2019structured}
Kaushik~Koneripalli Seetharam.
\newblock {\em Structured Disentangling Networks for Learning Deformation
  Invariant Latent Spaces}.
\newblock PhD thesis, Arizona State University, 2019.

\bibitem{srivastava2016functional}
Anuj Srivastava and Eric Klassen.
\newblock {\em {Functional and Shape Data Analysis}}.
\newblock Springer Series in Statistics. Springer New York, 2016.

\bibitem{srivastava2010shape}
Anuj Srivastava, Eric Klassen, Shantanu~H Joshi, and Ian~H Jermyn.
\newblock Shape analysis of elastic curves in euclidean spaces.
\newblock {\em IEEE Transactions on Pattern Analysis and Machine Intelligence},
  33(7):1415--1428, 2010.

\bibitem{su2014statistical}
Jingyong Su, Sebastian Kurtek, Eric Klassen, Anuj Srivastava, et~al.
\newblock Statistical analysis of trajectories on riemannian manifolds: bird
  migration, hurricane tracking and video surveillance.
\newblock {\em The Annals of Applied Statistics}, 8(1):530--552, 2014.

\bibitem{younes1998computable}
Laurent Younes.
\newblock Computable elastic distances between shapes.
\newblock {\em SIAM Journal on Applied Mathematics}, 58(2):565--586, 1998.

\end{thebibliography}

\end{document}